\documentclass{article}

\usepackage[preprint, nonatbib]{neurips_2021}

\usepackage[utf8]{inputenc} 
\usepackage[T1]{fontenc}    
\usepackage{hyperref}       
\usepackage{url}            
\usepackage{booktabs}       
\usepackage{amsfonts}       
\usepackage{nicefrac}       
\usepackage{microtype}      
\usepackage{xcolor}         
\usepackage{optidef}
\usepackage{bbold}
\usepackage{enumerate}
\usepackage{wrapfig}
\usepackage[ruled,vlined]{algorithm2e}
\usepackage{algorithmic}
\usepackage{amsthm}
\usepackage{caption}

\DeclareCaptionFont{mysize}{\fontsize{8}{9.6}\selectfont}
\newtheorem{prop}{Proposition}

\newtheorem{lemma}{Lemma}
\newtheorem{corollary}{Corollary}
\newtheorem{definition}{Definition}

\title{Maximum Entropy Weighted Independent Set Pooling for Graph Neural Networks}

\author{%
  Amirhossein Nouranizadeh\thanks{Both authors have equally contributed.} \\
  Department of Computer Engineering\\
  Amirkabir University of Technology\\
  Tehran, Iran \\
  \texttt{nouranizadeh@aut.ac.ir} \\
  \And
  Mohammadjavad Matinkia\footnotemark[1] \\
  Department of Computer Engineering\\
  Amirkabir University of Technology\\
  Tehran, Iran \\
  \texttt{matinkia@aut.ac.ir} \\
  \And
  Mohammad Rahmati\\
  Department of Computer Engineering\\
  Amirkabir University of Technology\\
  Tehran, Iran \\
  \texttt{rahmati@aut.ac.ir} \\
  \And
  Reza Safabakhsh\\
  Department of Computer Engineering\\
  Amirkabir University of Technology\\
  Tehran, Iran \\
  \texttt{safa@aut.ac.ir} \\
}

\begin{document}

\maketitle

\begin{abstract}
  In this paper, we propose a novel pooling layer for graph neural networks based on maximizing the mutual information between the pooled graph and the input graph. Since the maximum mutual information is difficult to compute, we employ the \textit{Shannon capacity of a graph} as an inductive bias to our pooling method.
  More precisely, we show that the input graph to the pooling layer can be viewed as a representation of a noisy communication channel.
  For such a channel, sending the symbols belonging to an independent set of the graph yields a reliable and error-free transmission of information.  
  We show that reaching the maximum mutual information is equivalent to finding a maximum weight independent set of the graph where the weights convey entropy contents.
  Through this communication theoretic standpoint, we provide a distinct perspective for posing the problem of graph pooling as maximizing the information transmission rate across a noisy communication channel, implemented by a graph neural network.
  We evaluate our method, referred to as Maximum Entropy Weighted Independent Set Pooling (MEWISPool), on graph classification tasks and the combinatorial optimization problem of the maximum independent set. Empirical results demonstrate that our method achieves the state-of-the-art and competitive results on graph classification tasks and the maximum independent set problem in several benchmark datasets.
\end{abstract}

\section{Introduction}
\label{sec:intro}
Graph neural networks (GNN) provide an efficacious tool for representation learning on graph-structured data. GNNs were initially developed to extend the original convolutional networks to graphs. However, these classes of neural networks were gradually enhanced both operationally and theoretically, for instance by leveraging attention and gating mechanisms\cite{benchmarking}.

The theoretical and practical advantages of hierarchical representation learning initiate the urge for pooling methods in graph domain. In contrast with the image domain where the grid structure of the data can be exploited to simply design a pooling strategy, the irregularity of the graph-structured data poses further complications for the pooling task. Analogous to the pooling operator in convolutional neural networks, increasing the receptive fields of the computational neurons to capture global dependencies and attenuating the superfluous information, comprise the primary motivations for adopting pooling layers in GNNs. Various pooling methods aim to encapsulate the node-level statistics of the graph into a possibly smaller yet informative graph while preserving the structural content \cite{hamilton}.

In this paper we propose a novel pooling layer referred to as \textit{Maximum Entropy Weighted Independent Set Pooling} (MEWISPool) which can be incorporated into GNNs in an end-to-end manner to realize effective hierarchical representations. The overarching objective of MEWISPool is to select a proper subset of nodes of the input graph which has the highest mutual information with the primary graph. We fulfill this goal using communication and information theoretic concepts. 

As a key element in the proposed pooling method, Shannon capacity of a graph measures the Shannon capacity of a noisy communication channel represented by the graph \cite{shannon, lovasz}. This graph is called the \textit{confusability graph} of the channel \cite{confusability}, where the nodes represent the symbols which are to be transmitted through the channel and the edges indicate the symbols which might be confused at the output due to the presence of noise. On the other hand, Shannon capacity of a noisy channel is also defined as the maximum information rate that can be transmitted through the channel and is given by the maximum mutual information between the channel's input and output \cite{elements}.

Through this communication theoretic standpoint, we provide a distinct perspective for posing the problem of pooling in GNNs as maximizing the information transmission rate across a noisy communication channel, which is implemented by a GNN. In fact, this GNN is a deterministic realization of the noisy channel at its highest information transmission rate. Consequently, MEWISPool makes the assumption that the input graphs to the pooling layer manifest representations of noisy communication channels and selects the nodes contributing to the Shannon capacity of such channels. This assumption is justified by the recent studies on the over-smoothing behavior of GNNs, which demonstrate that such networks act as low-pass filters and hence, result in smooth signals over the graph \cite{oversmoothing, oversmoothing2, pairnorm, lowpass, fastgcn}. This implies that the neighboring nodes carry similar signals due to the smoothness of the signals over the graph. Thus, the input graph to the pooling layer can be viewed as a confusability graph where the neighboring nodes contain similar informational contents. Furthermore, given a confusability graph of a noisy communication channel, sending the symbols belonging to an independent set of the graph yields a reliable and error-free transmission of information. \cite{shannon}. 

In this work, we employ the intuitions discussed above and strive to find an independent set of nodes whose signals have the maximum mutual information with the entire set of nodes signals. In order to circumvent the difficulties of solving the maximum mutual information, we utilize the \textit{Infomax principle} \cite{haykin, linkster, infomax} to transform the problem into maximizing the entropy of output signals given that their corresponding nodes form an independent set. Finally, we show that this problem is equivalent to finding a maximum weight independent set (MWIS) of the graph where the weights convey entropy contents. Solving the aforementioned MWIS problem maximizes the information transmission rate which is equal to the maximum mutual information between the input and output, and realizes the Shannon capacity of the input graph.

As MWIS is considered an NP-hard problem, finding the exact solution of the problem is intractable. However, we present a neural estimator to assign a probability to each node to be included in the optimal solution. The sub-optimal solution is then excluded from the learned distribution using derandomization algorithms. We verify that the existence of such a solution is guaranteed using \textit{the probabilistic method} \cite{erdos1959graph, probabilisticmethod}. Since MEWISPool samples the vertices following the sub-optimal solution of the MWIS problem, it does not require a pooling ratio which in turn, makes it adaptive to the graph structure and easier to fine-tune in end-to-end configurations. We evaluate the proposed method on graph classification tasks and the maximum independent set (MIS) problem. 

The contributions of the present work can be summarized as, (i) devising a novel objective for pooling, based on the concepts of Shannon capacity of graphs and Infomax principle, (ii) proposing a neural execution of the combinatorial optimization (CO) problem of MWIS, (iii) introducing a pooling layer, MEWISPool, with structure-adaptive pooling ratio, and (iv) demonstrating the state-of-the-art and competitive results on graph classification tasks and the MIS problem\footnote{\url{https://github.com/mewispool/mewispool}}.

\section{Related Works}
\label{sec:relwork}
\textbf{GNNs and pooling techniques.} In recent years, there has been a tremendous attention on GNNs due to their achievements in learning tasks on graph-structured data in various domains such as social sciences\cite{gcn}, bioinformatics\cite{messagepassing}, physics \cite{physics2}, recommendation systems\cite{recom2}, etc. GNNs utilize graph-based pooling techniques to learn hierarchical representations of the input graphs. 
In a general perspective, one can categorize the graph pooling techniques into the methods which are based on node selection \cite{gunet, sagpool, gxn, tagss}, and the methods based on graph coarsening \cite{diffpool, mincutpool, ecc, ipool, structpool}. The node-selection-based pooling methods assign an importance score to each node of the graph and select the high score nodes as the pooled nodes. On the other hand, coarsening-based methods tend to cluster the graph and merge each cluster into one node. 

Among the pooling methods, iPool \cite{ipool} and VIPool \cite{gxn}, utilize information theoretic concepts to select candidate nodes. More precisely, iPool defines an information gain criterion to quantitatively measure the conditional entropy of each node given its neighbors and selects the nodes with highest information gains. Alternatively, VIPool computes the mutual information between the neighboring nodes and selects the nodes which have the highest mutual information with their neighbors. Instead of considering the informational relationship between a node and its neighbors, MEWISPool tends to maximize the mutual information between the pooled nodes and the entire set of input vertices. Moreover, using the maximum entropy weighted independent set, MEWISPool guarantees a fair coverage of the graph, whereas iPool and VIPool might select the nodes only from a locality of the graph where the informational criteria are satisfied.

\textbf{Mutual information maximization.}
Recently, several research has been focused on mutual information estimation. MINE \cite{MINE} develops a general framework for maximization of mutual information, and introduces a consistent method for unsupervised representation learning. Similarly, \cite{DIM} introduces Deep InfoMax (DIM), for unsupervised representation learning while putting more focus on the intrinsic structure of image data. \cite{DGInfo} proposes the Deep Graph Infomax (DGI) which takes the graph-structured data into account and adopts the ideas from DIM to the graph domain. Also, InfoGraph \cite{infograph} maximizes the mutual information between the graph-level representation and the representations of substructures of the input graph.

In contrast to VIPool, which is heavily built upon the method of MINE, MEWISPool maximizes the mutual information with a completely different method by explicitly incorporating the structure of the input graphs. Indeed, MEWISPool takes this fact into account that the structure of the graph specifies how the mutual information between the pooled nodes and original nodes can be maximized. In other words, the MEWISPool's perspective to view the input graphs as the representations of noisy communication channels, enables it to directly exploit the structure of the input graphs to maximize the information transmitted through the pooling layer.

\textbf{Over-smoothing in GNNs.}
There are numerous reports emphasizing on the characteristic of GNNs in \textit{over-smoothing} the features, which results in gradual decrease in the performance of GNNs by increasing the number of layers \cite{oversmoothing, oversmoothing2}. This decay is partly attributed to over-smoothing, where repeated graph convolutions eventually make node embeddings indistinguishable \cite{pairnorm}.
\cite{lowpass} develops a theoretical framework based on graph signal processing to indicate that GNNs only perform low-pass filtering on feature vectors and do not have the nonlinear manifold learning property. Similarly, \cite{fastgcn} demonstrates that such networks correspond to fixed low-pass filters. Based on this concept, MEWISPool views the input graphs as the confusability graphs of noisy communication channels. The definition of the confusability graph will be provided in the supplementary materials.

\textbf{Neural execution of combinatorial optimization problems.}
Most of the CO problems including the MWIS problem, are considered to be NP-hard. However, the study of such problems from a neural perspective has become an engaging area of research \cite{co3, co4, co6, co5, co1, co2}. S2V-DQN \cite{dqn} combines reinforcement learning and graph embedding techniques to solve CO problems using deep Q-learning. \cite{guidedtree} utilizes GNNs in combination with guided tree search, to solve certain CO problems on graphs in a supervised setting. Further examples for neural execution of CO problems can be found in \cite{horizon}. More recently, \cite{erdos} has proposed a global framework for combinatorial optimization on graphs in unsupervised settings. Their framework uses GNNs to learn distributions on the nodes of the graphs. They use the probabilistic method to guarantee the existence of a valid solution which is then extracted using the derandomization method of conditional expectation. As MEWISPool relies on extracting the MWIS of a graph, we adopt the intuition behind the work of \cite{erdos} to provide an approximate solution to the problem via GNNs.

\section{Proposed Method}
\label{sec:prop_met}
In a general perspective, MEWISPool tends to pool a subset of nodes whose signals have the maximum mutual information with the entire set of input nodes signals while taking the structure of the input graph into account. Since the input graph signals to the pooling layer are the outputs of previous graph convolutional layers, due to the over-smoothing behavior of GNNs  \cite{lowpass, fastgcn, oversmoothing, oversmoothing2, pairnorm}, such signals mostly contain low-frequency components and exhibit smoothness in their values over the structure of the graph. Hence, the adjacent nodes signals might be confused due to low variation between their signal contents. 

Given the smoothness of neighboring signals, one can view the input graph as a noisy communication channel whose nodes signals are the messages which are to be transmitted and are prone to confusion due to the presence of noise. In this graph, the edges denote the pairs of messages which might be confused through transmission. 
The resulting graph is called the \textit{confusability graph} \cite{confusability} of the channel. In this perspective, sending the messages over an independent set of the graph yields a reliable and error-free transmission of messages \cite{shannon}. MEWISPool exploits this intuition and provides a neural implementation of the pooling layer as a reliable transmission of information across the channel while maximizing the mutual information. Here, we give a formal definition of MEWISPool's functionality. Note that the proofs of propositions \ref{prop1}, \ref{prop2}, \ref{prop4}, and \ref{prop5} are given in the supplementary materials. Also, the preliminaries on MIS and MWIS are presented in the supplementary materials.

Let $\mathcal{G}=(\mathcal{V}, \mathcal{E})$ be an undirected simple graph, where $\mathcal{V}$ denotes the set of vertices, and $\mathcal{E}$ denotes the set of edges, and let $\mathbf{X} \in \mathbb{R}^{|\mathcal{V}|\times d}$ be the feature matrix of the nodes of the graph. The objective of MEWISPool is to find an optimal subset of nodes, $\mathcal{U}^* \subset \mathcal{V}$, satisfying the following optimization problem:
\begin{figure}[!t]
	\vskip 0.2in
	\begin{center}
		\centerline{\includegraphics[width=\columnwidth]{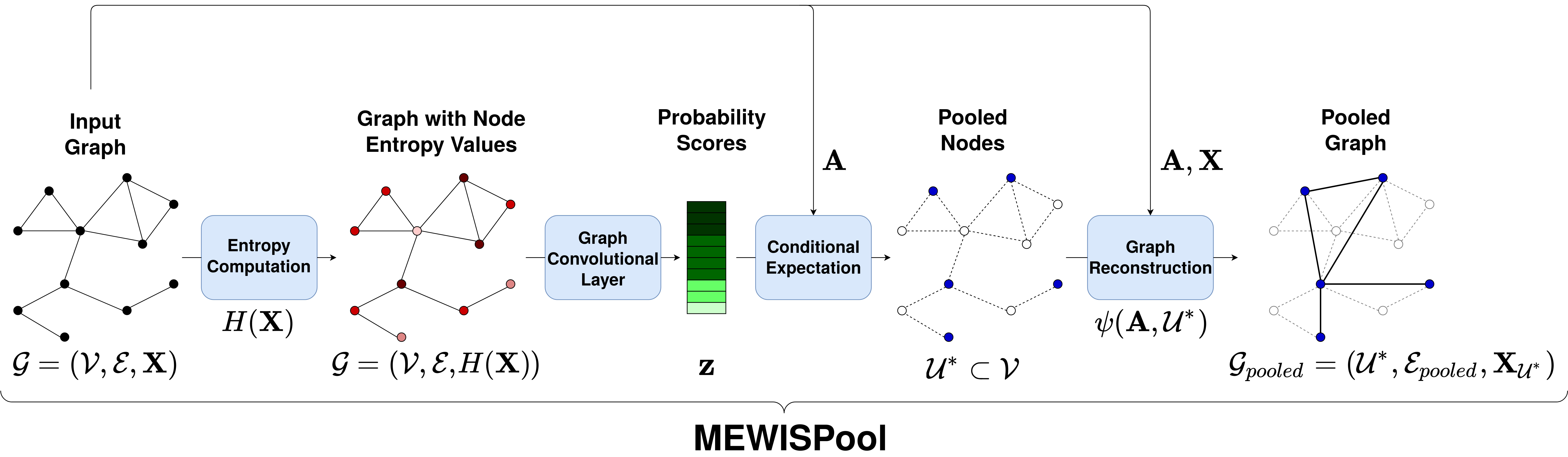}}
		\caption{\scriptsize Maximum Entropy Weighted Independent Set Pooling (MEWISPool) diagram. As the diagram shows, the input graph and its corresponding node features $\mathbf{X}$ are used to compute the entropy values of the nodes, $H(\mathbf{X})$ based on Eq.(\ref{eq:entropy}). Then the graph with node entropies as node features is fed to a GNN to generate a probability score vector $\mathbf{z}$, each of which elements represent the degree of membership of a node to the MWIS. The conditional expectation algorithm (Algorithm \ref{tab:1}) is then applied to the probability score along with the adjacency matrix of the graph to extract the maximum entropy weighted independent set of the graph, $\mathcal{U}^*$. Finally, the candidate nodes $\mathcal{U}^*$ and the adjacency matrix are utilized to reconstruct the pooled graph according to Eq.(\ref{eq:adj_pooled}). The feature vector of each node in the pooled graph is set the same as the feature vector of the same node in the input graph.}
		\label{fig: mewispool}
	\end{center}
	\vskip -0.2in
\end{figure}
\begin{maxi}|l|
	{\mathcal{U}}{\mathcal{I}(\mathbf{X};\mathbf{X}_\mathcal{U})}{}{}
	\label{eq:mewis1}
	\addConstraint{(i, j) \notin \mathcal{E}}{,}{\forall i, j \in \mathcal{U}},
\end{maxi}

where $\mathcal{I}(\mathbf{X};\mathbf{X}_\mathcal{U})$ is the mutual information between $\mathbf{X}$ and $\mathbf{X}_\mathcal{U}$, the signals over the subset of nodes $\mathcal{U} \subset \mathcal{V}$. Also, $i$ and $j$ denote the nodes of the graph and $(i, j)$ denotes an edge between the nodes $i$ and $j$. The mutual information between two random variables $X$ and $Y$ with the distributions $p_X$ and $p_Y$ is defined as the Kullback-Leibler divergence between their joint distribution and the product of their marginal distributions, i.e., $\mathcal{I}(X;Y) = D_{KL}(p_{(X, Y)}\| p_X \otimes p_Y)$. The solution of the optimization problem \ref{eq:mewis1} represents a subset $\mathcal{U}^*$ of nodes whose corresponding signals have the maximum mutual information with the set of the signals on the entire graph, while the resulting selected nodes are mutually disconnected. In this work, we aim to find the optimal solution $\mathcal{U}^*$ using GNNs.

In order to circumvent the difficulties of solving problem \ref{eq:mewis1}, we reformulate it to
\begin{maxi}|l|
	{\mathcal{U}}{H(\mathbf{X}_\mathcal{U})}{}{}
	\label{eq:mewis2}
	\addConstraint{(i, j) \notin \mathcal{E}}{,}{\forall i, j \in \mathcal{U}},
\end{maxi}
where $H(\mathbf{X}_\mathcal{U})$ is the entropy of the signals over the nodes of $\mathcal{U}$.\smallskip

\begin{prop}\label{prop1}
	Let $\mathbf{X}$ be a set of random variables and $\mathbf{X}_\mathcal{U}$ be a subset of $\mathbf{X}$ denoted by the set of indices $\mathcal{U}$. Let $\phi$ be a neural network, mapping $\mathbf{X}$ to $\mathbf{X}_\mathcal{U}$. Then maximizing the mutual information $\mathcal{I}(\mathbf{X};\mathbf{X}_\mathcal{U})$ is equivalent to maximizing the joint entropy $H(\mathbf{X}_\mathcal{U})$.
\end{prop}
Formulation of problem \ref{eq:mewis2} is based on the Infomax principle \cite{linkster, haykin, infomax} which states that in order to maximize the mutual information between the input and output of a neural network, one can maximize the entropy of the neural network's output. This formulation translates to finding an independent set of the graph with maximum joint entropy. Further, we show that the problem \ref{eq:mewis2} can be reduced to
\begin{maxi}|l|
	{\mathcal{U}}{\sum_{i \in \mathcal{U}} H(\mathbf{X}_i)}{}{}
	\label{eq:mewis3}
	\addConstraint{(i, j) \notin \mathcal{E}}{,}{\forall i, j \in \mathcal{U}}.
\end{maxi}

\begin{prop}\label{prop2}
	Let $\mathcal{G} = (\mathcal{V}, \mathcal{E})$ be an undirected simple graph and $\mathbf{X}$ be the set of random variables assigned to the nodes of $\mathcal{G}$. Also, let $g$ be a probability distribution associated to $\mathbf{X}$ such that for any $(i, j) \notin \mathcal{E}$, it follows that $\mathbf{X}_i$ and $\mathbf{X}_j$ are independent. Then, problems \ref{eq:mewis2} and \ref{eq:mewis3} are equivalent.
\end{prop}

\begin{prop}\label{prop3}
	The problem \ref{eq:mewis3} is the definition of the maximum weight independent set problem where the weight of node $i$ is defined as $H(\mathbf{X}_i)$.
\end{prop}
\begin{proof}
	Problem \ref{eq:mewis3} states that we want to select as many nodes as possible to maximize the sum of nodes entropies while no two of the selected nodes are adjacent. Obviously, this is the definition of the MWIS problem where the weights are defined as $H(\mathbf{X}_i)$.
\end{proof}

For solving the problem \ref{eq:mewis3}, first we need to introduce the probability distribution $g$ over the nodes signals which satisfies the condition of proposition \ref{prop2}. To this end, we refer to the concept of over-smoothing in GNNs. Intuitively, since the graph signals are smooth with respect to the graph structure, the presence of a node signal with high variation with respect to its neighboring signals is less probable. Hence, we model the probability distribution of the nodes signals inversely proportional to the variations of nodes signals with respect to their neighboring signals. Mathematically speaking, for an undirected and unweighted simple graph $\mathcal{G}$, and its corresponding graph signal $\mathbf{x} \in \mathbb{R}^{|\mathcal{V}|}$ the \textit{local variation} $\delta_i$ at vertex $i$ is defined as,
\begin{gather}
	\delta_i = \left[\sum_{j \in \mathcal{N}_i}(x_j - x_i)^2\right]^\frac{1}{2},
\end{gather}
where $\mathcal{N}_i$ represents the neighbors of the node $i$ \cite{emergingfields}. The local variation provides a measure of the smoothness of the graph signal $\mathbf{x}$ around a vertex. For the case of $d$-dimensional graph signal $\mathbf{X} \in \mathbb{R}^{|\mathcal{V}|\times d}$, we calculate the local variations at each dimension, and take the $\textrm{L}_2$ norm of the local variations vector as the variation of each node:
\begin{gather}\label{eq:Delta}
	\Delta_i = \left[\sum_{k=1}^d(\delta_i^{(k)})^2\right]^\frac{1}{2},
\end{gather}
where $\delta_i^{(k)}$ is the local variation of the node $i$ at the dimension $k$. As mentioned before, the input graph signal to MEWISPool is a smooth graph signal, meaning that the local variation of each node is supposedly small. Hence, the occurrence of a node with relatively high local variation is less probable. We mathematically model this intuition as,
\begin{gather}
	\label{eq:var_prob}
	p(\mathbf{X}_i) \propto \exp(-\Delta_i),
\end{gather}
where $\mathbf{X}_i \in \mathbb{R}^d$ is the signal on the node $i$. Based on Eq.(\ref{eq:var_prob}) we propose a notion of node entropy by assigning a probability $p(\mathbf{X}_i)$ to the node $i$ according to
\begin{gather}\label{eq:softmax}
	p(\mathbf{X}_i) = \frac{\exp(-\Delta_i)}{\sum_{j=1}^{|\mathcal{V}|} \exp(-\Delta_j)},
\end{gather}
\begin{algorithm}[tb]
	\caption{Maximum Entropy Weighted Independent Set Extraction}
	\label{tab:1}
	\begin{algorithmic}[1]
		\REQUIRE Adjacency Matrix $\boldsymbol{A} \in \mathbb{R}^{|\mathcal{V}| \times |\mathcal{V}|}$, Probability Scores $\mathbf{z} = [z_1,\dots,z_{|\mathcal{V}|}]^T$, Fixed Parameter $\gamma = \sum_{i}H(\mathbf{X}_i)$, Fixed Threshold $T = \mathcal{L}_{\textrm{pool}}(\mathcal{G};\theta)$.
		\ENSURE Maximum Entropy Weighted Independent Set Approximation.
		\STATE \textbf{Initialize} \textit{selected} $ = \emptyset$, and \textit{rejected}  $ = \emptyset$
		\FOR{each node $v$ in descending order of probability scores $\mathbf{z}$}
		\IF{$v \not\in selected \cup rejected$}
		\STATE $\mathbf{s} \leftarrow \mathbf{z}$
		\STATE $\mathbf{s}_v \leftarrow 1$
		\STATE $\mathbf{s}_{\mathcal{N}_v} \leftarrow 0$
		\IF{$\gamma -  \sum_{i=1}^{|\mathcal{V}|} H(\mathbf{X}_i)s_i + \sum_{(i, j) \in \mathcal{E}}s_is_j \leq T$}
		\STATE $selected \leftarrow v$
		\STATE $rejected \leftarrow \mathcal{N}_v$
		\STATE $\mathbf{z}_v \leftarrow 1$
		\STATE $\mathbf{z}_{\mathcal{N}_v} \leftarrow 0$
		\ENDIF
		\ENDIF
		\ENDFOR
		\STATE \textbf{return} $selected$
	\end{algorithmic}
\end{algorithm}
which is the \textit{softmax} function of the vector $-\Delta = [-\Delta_1, \dots, -\Delta_{|\mathcal{V}|}]^T$. 
\begin{prop}\label{prop4}
	The probability distribution of Eq.(\ref{eq:softmax}) satisfies the condition of distribution $g$ of proposition \ref{prop2}.
\end{prop}
Next, we derive the node entropy as
\begin{gather}
	\label{eq:entropy}
	H(\mathbf{X}_i) = -p(\mathbf{X}_i)\log(p(\mathbf{X}_i)).
\end{gather}
Based on proposition \ref{prop3}, using GNNs, MEWISPool approximately solves the MWIS problem with weights as defined in Eq.(\ref{eq:entropy}). To derive the appropriate loss function for MEWISPool, we adopt the framework proposed in \cite{erdos}, which is based on the \textit{probabilistic method} \cite{erdos1959graph, probabilisticmethod}. More precisely, MEWISPool assigns probability scores to the nodes of a graph which represent the degree of membership of the nodes in the MWIS. Further, MEWISPool efficiently extracts the maximum entropy weighted independent set according to the computed probability scores and using the \textit{method of conditional expectation} \cite{raghavan}. Here, we only present the derived loss function and algorithm, and we elaborate the details in supplementary materials. MEWISPool utilizes a GNN which takes a graph and its corresponding nodes entropy weights as inputs and is trained by minimizing the following loss function,
\begin{gather}
	\label{eq:6}
	\mathcal{L}_{\textrm{pool}}(\mathcal{G};\theta) = \gamma - \sum_{i=1}^{|\mathcal{V}|} H(\mathbf{X}_i)z_i + \sum_{(i, j) \in \mathcal{E}}z_iz_j,
\end{gather}
where $\theta$ denotes the parameters of the GNN and $z_i$ is the probability score assigned to the node $i$ generated by the GNN. Also, $\gamma$ is a fixed parameter and is set to $\gamma = \sum_{i}H(\mathbf{X}_i)$ to ensure a positive loss value.

Once the probability scores are generated by the GNN, MEWISPool applies \textbf{Algorithm} \ref{tab:1} to incrementally extract the maximum entropy weighted independent set of the input graph. Briefly, Algorithm \ref{tab:1} iterates through the nodes of the graph. In each step, if the current node is neither selected nor rejected, the algorithm calculates the value of Eq.(\ref{eq:6}) conditioned on the selection of the current node and rejection of its neighboring nodes. If the calculated value is less than or equal to the value of the MEWISPool's loss function, the algorithm selects the current node and rejects its neighbors. This procedure is based on the method of conditional expectation which is elaborated in the supplementary materials.

Thus far, we have shown how to solve problem \ref{eq:mewis3} and select the pooled nodes $\mathcal{U}^*$. After the nodes and their corresponding graph signals are pooled based on the solution of problem \ref{eq:mewis3}, the adjacency matrix for the pooled nodes is constructed according to
\begin{gather}
	\label{eq:adj_pooled}
	\mathbf{A}_{\textrm{pooled}} = \psi(A, \mathcal{U}^*),
\end{gather}
where the function $\psi$ is defined as $\psi(A, \mathcal{U}^*) = (\mathbb{1}_{|\mathcal{U}^*|} - \mathbf{I}_{|\mathcal{U}^*|})\odot\textrm{clip}(\mathbf{A}^2_{\mathcal{U}^*} + \mathbf{A}^3_{\mathcal{U}^*})$. $\mathbb{1}_{|\mathcal{U}^*|}$ is a square all-ones matrix of size $|\mathcal{U}^*|$ and $\mathbf{I}_{|\mathcal{U}^*|}$ is the identity matrix of size $|\mathcal{U}^*|$. $\mathbf{A}^k_{\mathcal{U}^*}$ is the submatrix of $\mathbf{A}^k$ corresponding to the nodes of $\mathcal{U}^*$. The term $(\mathbb{1}_{|\mathcal{U}^*|} - \mathbf{I}_{|\mathcal{U}^*|})$, simply removes any emerging self-loops. Finally, the $\textrm{clip}$ function clips any values greater than $1$ to $1$. We justify the choice of the reconstruction function of Eq.(\ref{eq:adj_pooled}) by the following proposition.

\begin{prop}\label{prop5}
	For a simple graph $\mathcal{G} = (\mathcal{V}, \mathcal{E})$ with an associated adjacency matrix $\mathbf{A}$ and an arbitrary maximal independent set $\mathcal{M}$, the reconstruction function of Eq.(\ref{eq:adj_pooled}), preserves the connectivity of each connected component of the graph.
\end{prop}
To recapitulate, MEWISPool performs as follows; It extracts the pooled nodes and their signals, $\mathbf{X}_{\textrm{pooled}}$, by solving problem \ref{eq:mewis3} and subsequently, reconstructs the pooled adjacency matrix based on Eq.(\ref{eq:adj_pooled}). The complete procedure of MEWISPool as a pooling layer is depicted in Figure \ref{fig: mewispool}.

For graph classification tasks, the whole network is trained according to $\mathcal{L}=\mathcal{L}_{\textrm{task}}+\alpha \mathcal{L}_{\textrm{pool}}$, where $\mathcal{L}_{\textrm{task}}$ is the cross-entropy loss for classification, and $\alpha$ is a regularization factor whose effect is studied in the supplementary materials. Technically, $\mathcal{L}_{\textrm{task}}$ is defined as $-\mathbf{y}^T\log (\mathbf{\hat{y}})$, where $\mathbf{y}$ and $\mathbf{\hat{y}}$ are the ground truth and predicted class labels, respectively.

\begin{figure*}[t]
	\vskip 0.2in
	\begin{center}
		\centerline{\includegraphics[width=\columnwidth]{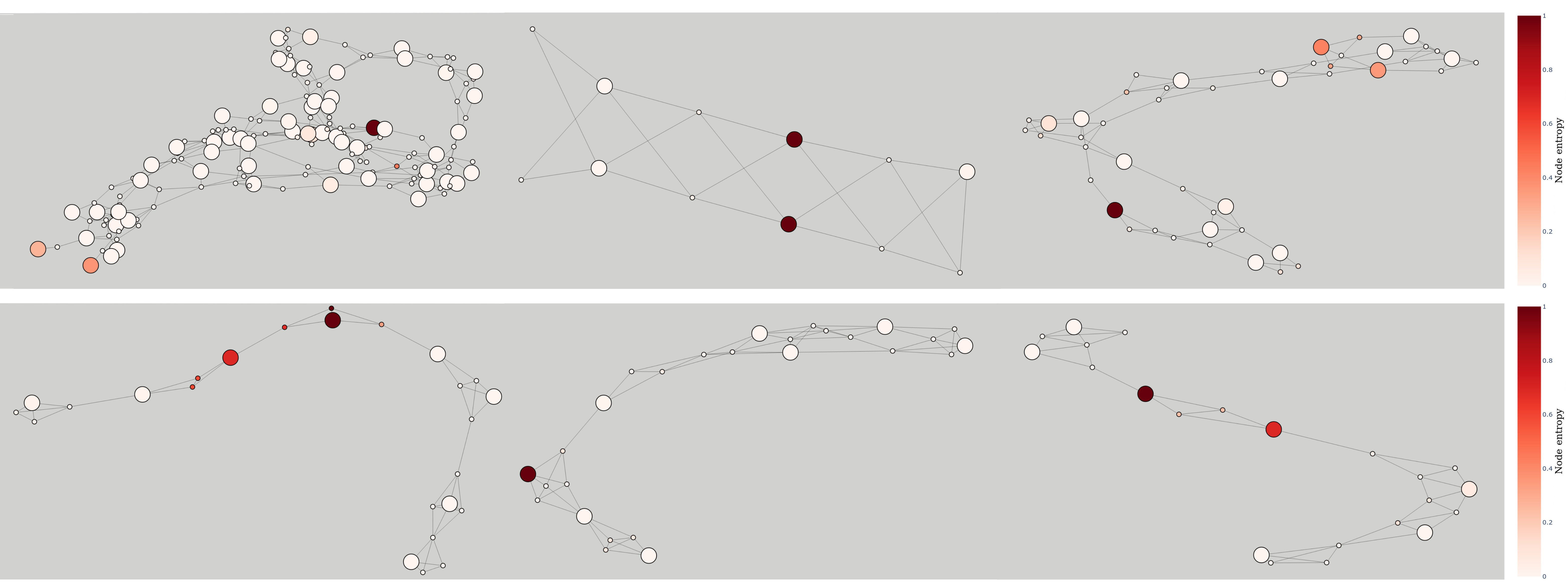}}
		\caption{\scriptsize Illustration of sampled nodes by MEWISPool in PROTEINS dataset\cite{dd1prot,enzprot}. The higher color intensities correspond to higher node entropies. Nodes with larger size represent the nodes sampled by MEWISPool. Note that MEWISPool adaptively samples nodes that cover the entire graph.}
		\label{fig:enz}
	\end{center}
	\vskip -0.2in
\end{figure*}

\section{Experiments}
\label{sec:experiments}
In this section, we evaluate our method on several problems, namely supervised graph classification task, and the unsupervised MIS problem. Furthermore, we introduce different architectural settings used on each task and dataset. We also mention the advantage of MEWISPool comparing to other pooling techniques in being adaptive, so that it does not require the hyper-parameter of pooling ratio.

\subsection{Datasets}\label{subsec:datasets}
To evaluate MEWISPool on graph classification tasks, we use the social network datasets \textbf{IMDB-B}, \textbf{IMDB-M}, and \textbf{COLLAB} \cite{imdbcollab}, small molecules datasets \textbf{FRANKENSTEIN} \cite{frankenstein}, \textbf{Mutagenicity} \cite{mutagenicity1}, and \textbf{MUTAG} \cite{mutag1}, and also bioinformatics datasets \textbf{D\&D} \cite{dd2}, and \textbf{PROTEINS} \cite{dd1prot}. For the MIS problem we use the citation networks datasets \cite{citationnetworks}. The statistics of each dataset is reported in the supplementary materials.

\begin{table}[t]
	\caption{\scriptsize The empirical results of graph classification tasks. The best and second best results are demonstrated with black and blue bold-faced fonts, respectively. The results of MEWISPool consist of the means and standard deviations of the accuracies according to the 10-fold cross-validation experimental setting. Note that the accuracy values with no standard deviations have been reported in the same fashion as their corresponding papers. Also, the \textbf{-} signs in the table indicates that the result of the corresponding dataset has not been reported in the original papers.}
	\resizebox{\columnwidth}{!}{
		\begin{tabular}{lcccccccc}
			\toprule
			\large Model & \large D\&D & \large FRANKENSTEIN & \large MUTAG & \large Mutagenicity & \large PROTEINS & \large COLLAB & \large IMDB-B & \large IMDB-M \\
			\midrule
			\large Set2Set\cite{set2set} & \large $70.83 \pm 0.84$ & \large $61.92 \pm 0.73$ & \large $86.78 \pm 7.33$ & \large $71.35 \pm 2.1$ & \large $79.33 \pm 0.84$ & \large $71.75$ & \large $71.00 \pm 7.54$ & \large $49.73 \pm 4.19$ \vspace{0.2cm} \\
			\large ECC\cite{ecc} & \large $73.68$ & \large $63.87 \pm 2.02$ & \large $88.33$ & \large $71.89 \pm 1.20$ & \large $72.33 \pm 3.4$ & \large $67.82 \pm 2.4$ & \large $67.70 \pm 2.8$ & \large $43.48 \pm 3.0$ \vspace{0.2cm} \\
			\large SortPool\cite{dgcnn} & \large $79.37 \pm 0.94$ & \large $60.61 \pm 0.77$ & \large $85.83 \pm 1.66$ & \large $70.66 \pm 1.51$ & \large $76.26 \pm 0.24$ & \large $73.76 \pm 0.49$ & \large $70.03 \pm 0.86$ & \large $47.83 \pm 0.85$ \vspace{0.2cm} \\
			\large DiffPool\cite{diffpool} & \large $75.05 \pm 3.4$ & \large $60.60 \pm 1.62$ & \large $88.87 \pm 6.75$ & \large $77.6 \pm 2.7$ & \large $73.72 \pm 3.5$ & \large $74.83 \pm 2.0$ & \large $68.40 \pm 6.1$ & \large $45.62 \pm 3.4$ \vspace{0.2cm} \\
			\large StructPool\cite{structpool} & \large $\mathbf{\textcolor{blue}{84.19}}$ & \large - & \large $\mathbf{\textcolor{blue}{93.59}}$ & \large - & \large $\mathbf{\textcolor{blue}{80.36}}$ & \large $74.22$ & \large $74.70$ & \large $52.47$ \vspace{0.2cm} \\
			\large SAGPool\cite{sagpool} & \large $78.35 \pm 3.5$ & \large $62.57 \pm 0.60$ & \large $93.32 \pm 4.16$ & \large $72.4 \pm 2.4$ & \large $78.28 \pm 4.0$ & \large $76.92 \pm 1.6$ & \large $72.80 \pm 2.3$ & \large $49.43 \pm 2.6$ \vspace{0.2cm} \\
			\large Graph U-Net\cite{gunet} & \large $82.14 \pm 3.0$ & \large $61.46 \pm 0.84$ & \large $87.77 \pm 6.47$ & \large $71.9 \pm 3.7$ & \large $77.20 \pm 4.3$ & \large $77.58 \pm 1.6$ & \large $73.40 \pm 3.7$ & \large $50.27 \pm 3.4$ \vspace{0.2cm} \\
			\large iPool\cite{ipool} & \large $79.45 \pm 2.78$ & \large - & \large $90.42 \pm 4.68$ & \large - & \large $77.36 \pm 3.27$ & \large $77.28 \pm 2.17$ & \large $73.30 \pm 2.72$ & \large $51.27 \pm 3.44$ \vspace{0.2cm} \\
			\large MinCutPool\cite{mincutpool} & \large $80.8 \pm 2.3$ & \large $\mathbf{\textcolor{blue}{65.94 \pm 1.60}}$ & \large $87.34 \pm 6.31$ & \large $79.9 \pm 2.1$ & \large $76.5 \pm 2.6$ & \large $\mathbf{83.4 \pm 1.7}$ & \large $\mathbf{\textcolor{blue}{79.0 \pm 2.0}}$ & \large $52.8 \pm 1.69$ \vspace{0.2cm} \\
			\large VIPool\cite{gxn} & \large $82.68 \pm 4.1$ & \large - & \large - & \large $\mathbf{\textcolor{blue}{80.19 \pm 1.02}}$ & \large $79.91 \pm 4.1$ & \large $78.82 \pm 1.4$ & \large $78.60 \pm 2.3$ & \large $\mathbf{\textcolor{blue}{55.20 \pm 2.5}}$ \vspace{0.2cm} \\
			\midrule
			\large MEWISPool & \large $\mathbf{84.33 \pm 1.18}$ & \large $\mathbf{73.46 \pm 0.95}$ & \large $\mathbf{96.66 \pm 1.23}$ & \large $\mathbf{80.66 \pm 1.72}$ & \large $\mathbf{80.71 \pm 2.31}$ & \large $\mathbf{\textcolor{blue}{79.66 \pm 4.02}}$ & \large $\mathbf{82.13 \pm 1.21}$ & \large $\mathbf{56.23 \pm 1.04}$ \\
			\bottomrule
		\end{tabular}}
	\label{tab:graph_class}
\end{table}

\subsection{Experimental Setup}\label{subsec:exp-setup}
The implementation of MEWISPool consists of three graph convolutional layers for the graph classification task and six graph convolutional layers for the MIS problem. We use \textbf{GIN}\cite{gin} for the graph convolutional layers. Following the convolutional layers, the conditional expectation module, as explained in Algorithm \ref{tab:1}, extracts the MWIS of the graph. Note that in contrast with the majority of pooling layers, MEWISPool does not require a pooling ratio, which in turn facilitates the procedure by treating the ratio adaptively with respect to the graph structure. For the MIS experiments, the nodes weights are all set to $1$, implying equal importance for all the vertices.

In order to conduct the experiments for graph classification tasks, we use an architecture which consists of three GIN layers, two MEWISPool layers, and two dense layers which constitute the classifier. MEWISPool is followed by each convolutional layer except the last GIN which is followed by the classifier. Furthermore, we employ batch normalization\cite{batchnorm} and dropout\cite{dropout} layers after each convolutional layer. The activation functions for all neurons are set to the rectified linear units (ReLU).

The whole network is trained in an end-to-end manner using the \textbf{Adam}\cite{adam} optimizer. For the graph classification task, the objective is the minimization of the cross-entropy loss for the classification and the MEWISPool's loss as in Eq.(\ref{eq:6}), whereas in the MIS problem, the only objective is the loss function of MEWISPool. The trade-off between the cross-entropy loss and the MEWISPool's loss are controlled via a regularization factor which is set to $0.01$ for all the graph classification experiments. The learning rate is set to $10^{-3}$ and the model is trained for $200$ epochs. Additionally, all graph classification experiments are performed in a 10-fold cross-validation experimental setting. The whole experiment is implemented using \textbf{PyTorch}\cite{pytorch}, \textbf{PyTorch Geometric}\cite{torchgeom}, and \textbf{Deep Graph Library}\cite{dgl} packages, and is executed on GeForce GTX 1080 Ti GPU.

\subsection{Empirical Results}
\textbf{Supervised graph classification task.} In this experiment, we quantitatively evaluate the performance of MEWISPool. The datasets of this task contain sets of input graphs with their corresponding labels, and the task is to classify each input graph. The datasets mostly consist of feature vectors or node labels for the vertices, which we use as input features to the model. For the datasets with no node attributes, we use the one-hot encoding of node degrees as input features. We compare the performance of MEWISPool with the state-of-the-art pooling techniques, namely, \textbf{Set2Set}\cite{set2set}, \textbf{ECC}\cite{ecc}, \textbf{SortPool}\cite{dgcnn}, \textbf{DiffPool}\cite{diffpool}, \textbf{SAGPool}\cite{sagpool}, \textbf{Graph U-Net}\cite{gunet}, \textbf{MinCutPool}\cite{mincutpool}, \textbf{StructPool} \cite{structpool}, \textbf{iPool} \cite{ipool}, and \textbf{VIPool} \cite{gxn}. Table \ref{tab:graph_class} demonstrates the performance of MEWISPool in comparison with the aforementioned methods. As the results suggest, MEWISPool outperforms other methods in seven benchmark datasets and achieves the second best result in the COLLAB dataset. Figure \ref{fig:enz} illustrates sample maximum entropy weighted independent sets extracted by MEWISPool from the PROTEINS dataset.

\textbf{Unsupervised combinatorial optimization problem of MIS.} In this experiment, we evaluate a standalone implementation of MEWISPool, to approximately solve the NP-hard problem of MIS. The problem is to find the largest possible subset of vertices in the graph such that no two of which are adjacent. The evaluation metric for assessing the performance of MIS-solving methods is simply the cardinality of the solution set found by the method. Here, as in \cite{guidedtree}, we compare MEWISPool with a greedy classic heuristic referred to as \textbf{Classic}\cite{classic}, and two state-of-the-art methods for solving MIS problem, namely, \textbf{S2V-DQN}\cite{dqn}, which is based on reinforcement learning paradigms, and \textbf{GCN-GTS}\cite{guidedtree}, which is a supervised method based on the combination of GNNs and guided tree search heuristic algorithm. The numerical results are shown in Table \ref{tab:mis}. As the results suggest, MEWISPool outperforms the Classic and S2V-DQN methods but fails to perform as desirably as GCN-GTS. However, it is worth mentioning that GCN-GTS treats the problem in a supervised manner, meaning that it requires a set of training solutions for the MIS problem, while MEWISPool tries to solve the MIS problem with no supervision. Figure \ref{fig:loss-mis} illustrates the learning curve of MEWISPool and the evolution of solved MIS for the Cora dataset.
\begin{table}
	\begin{minipage}{0.40\linewidth}
		\caption{\scriptsize The comparison of MEWISPool's performance and other methods based on the size of the solution. MEWISPool consistently performs better than the classic greedy method and the reinforcement learning method of S2V-DQN \cite{classic, dqn}.}
		\resizebox{\columnwidth}{!}{
			\begin{tabular}{lccc}
				\toprule
				Model & Cora & Citeseer & PubMed \\
				\midrule
				Classic\cite{classic} & $1424$ & $1848$ & $15852$  \\
				S2V-DQN\cite{dqn} & $1381$ & $1705$ & $15705$ \\
				GCN-GTS\cite{guidedtree} & $\mathbf{1451}$ & $\mathbf{1867}$ & $\mathbf{15912}$ \\
				\midrule
				MEWISPool & $\mathbf{\textcolor{blue}{1433}}$ & $\mathbf{\textcolor{blue}{1852}}$ & $\mathbf{\textcolor{blue}{15862}}$ \\
				\bottomrule
		\end{tabular}}
		\label{tab:mis}
	\end{minipage}\hfill
	\begin{minipage}{0.55\linewidth}
		\centering
		\includegraphics[scale=0.08]{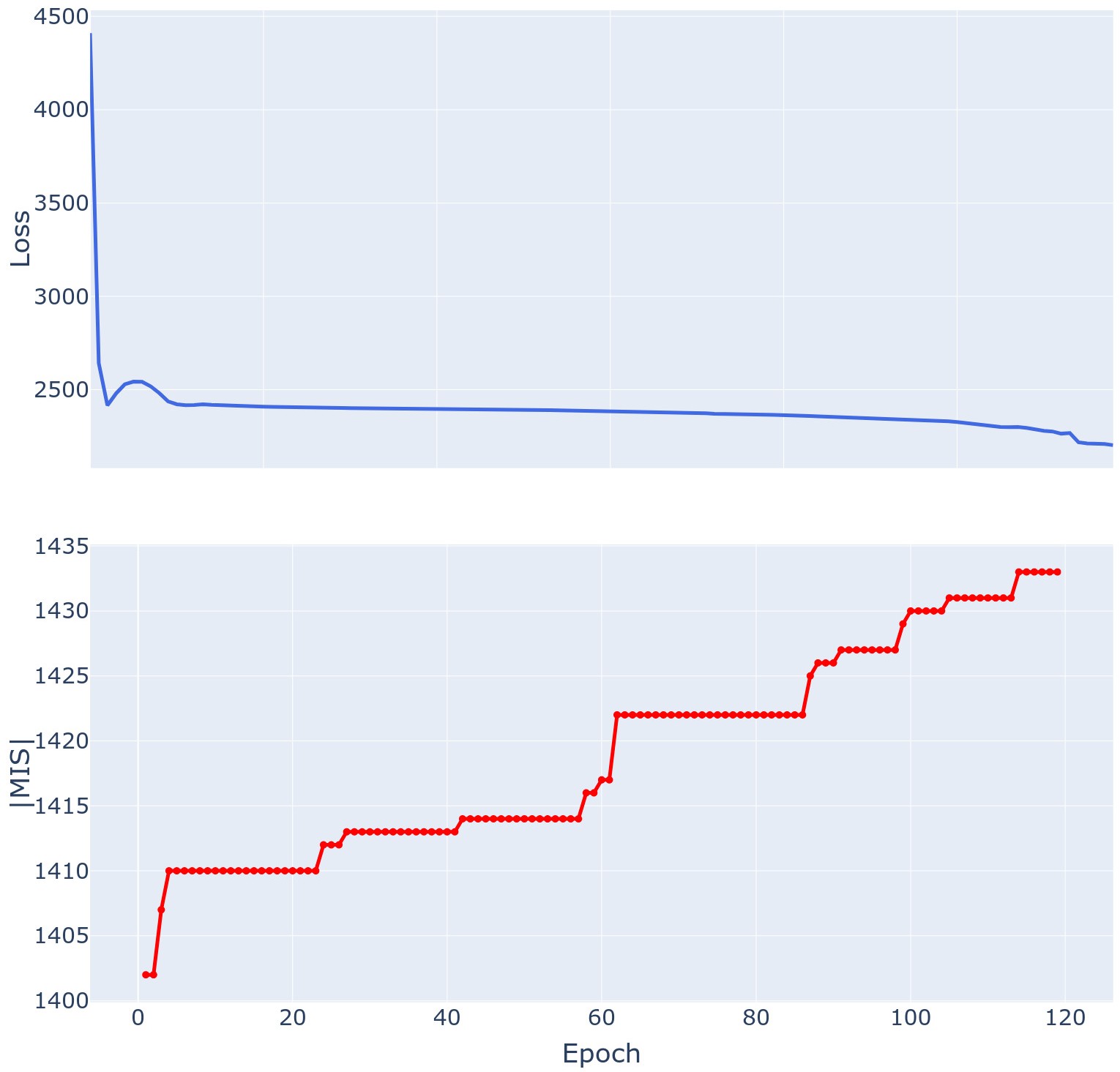}
		\captionof{figure}{\scriptsize Training loss (up) and size of learned MIS (down) for Cora dataset.}
		\label{fig:loss-mis}
	\end{minipage}
\end{table}

\subsection{Complexity Analysis}\label{subsec:ca}

\textbf{MIS problem.} For the MIS problem, the weights of all nodes are preset to $1$ and the computational complexity of MEWISPool reduces to the computational complexity of the GNN followed by the derandomization algorithm of conditional expectation as presented in Algorithm \ref{tab:1}. The computational complexity of a graph convolutional layer is $\mathcal{O}(|\mathcal{V}|^3)$ \cite{structpool}. In the worst case, Algorithm \ref{tab:1} iterates through all the nodes of the graph ($|\mathcal{V}|$), and for each node, it calculates the loss function of MEWISPool as in Eq.(\ref{eq:6}).
Dominated by the second and third terms, the complexity of the loss function is $\mathcal{O}(|\mathcal{V}|) + \mathcal{O}(|\mathcal{E}|)$. Therefore, the total complexity of the conditional expectation module is $\mathcal{O}(|\mathcal{V}|^2) + \mathcal{O}(|\mathcal{V}||\mathcal{E}|)$. Finally, the computational complexity of MEWISPool for the MIS problem can be written as $\mathcal{O}(|\mathcal{V}|^3) + \mathcal{O}(|\mathcal{V}|^2) + \mathcal{O}(|\mathcal{V}||\mathcal{E}|)$. Assuming that the graphs are sparse, i.e., $\mathcal{O}(|\mathcal{E}|) \approx \mathcal{O}(|\mathcal{V}|)$, the complexity reduces to $\mathcal{O}(|\mathcal{V}|^2) + \mathcal{O}(|\mathcal{V}|^3) \approx \mathcal{O}(|\mathcal{V}|^3)$ which is the complexity of a single graph convolutional layer. For the case of the Cora dataset, which contains 2708 nodes and 5429 edges, the running time of a feed-forward of the network is 2.9 seconds.

\textbf{Graph classification task.} 
The computational complexity of MEWISPool for the graph classification task is similar to the MIS problem, except that we need to compute entropy-based weights beforehand. Computation of node entropies involves calculating L$_2$ norms of the local variation vectors for the nodes (Eq.(\ref{eq:Delta})) with the order of $\mathcal{O}(|\mathcal{V}| k d)$, where $k$ is the maximum degree of the graph and $d$ is the dimensionality of nodes features.
Computing the probabilities of the nodes (Eq.(\ref{eq:softmax})) is in the order of $\mathcal{O}(|\mathcal{V}|)$, and finally computing the nodes entropies based on the calculated probabilities is in the order of $\mathcal{O}(|\mathcal{V}|)$. The rest is similar to the case of MIS problem. Thus, the ultimate complexity of MEWISPool for the graph classification task is $\mathcal{O}(|\mathcal{V}| k d) + \mathcal{O}(|\mathcal{V}|^3) + \mathcal{O}(|\mathcal{V}||\mathcal{E}|)$. Assuming $d \ll |\mathcal{V}|$ and the graphs are sparse, i.e., $\mathcal{O}(|\mathcal{E}|) \approx \mathcal{O}(|\mathcal{V}|)$ and $\mathcal{O}(k)$ is constant, the complexity reduces to $\mathcal{O}(|\mathcal{V}| d) + \mathcal{O}(|\mathcal{V}|^3) \approx \mathcal{O}(|\mathcal{V}|^3)$ which is the complexity of a single graph convolutional layer. For the case of the large scale D\&D dataset, the average running time of a feed-forward of the network for a single graph is 0.059 second.

Even though MEWISPool has a relatively high computational cost, its functionality involves finding an approximate solution for the NP-hard MWIS problem, however, due to its superior performance, it motivates further research on applying CO problems within the graph neural architectures.

\section{Conclusion}
\label{sec:conc}

Pooling operation serves as an essential component in GNNs for hierarchical representation learning. In this paper, we developed a structure-adaptive pooling layer based on the combinatorial optimization problem of MWIS with entropy-based weights. We expanded the insights regarding the proposed method using the concepts of Shannon capacity of graphs in communication theory and the Infomax principle for maximizing the mutual information between the network's input and output. We proposed a neural execution to approximate the solution for the NP-hard problem of MWIS. Finally, we evaluated the method on several benchmark datasets for graph classification tasks and the MIS problem, and achieved the state-of-the-art and competitive results.

\bibliographystyle{plain}

\appendix

\section{Appendix}
\subsection{The Maximum Independent Set and the Maximum Weight Independent Set of a Graph}
Let $\mathcal{G} = (\mathcal{V}, \mathcal{E})$ be an undirected, unweighted, and simple graph, where $\mathcal{V}$ is the set of vertices or nodes of the graph and $\mathcal{E}$ denotes the set of edges. An \textit{independent set} of the graph is defined as a subset of nodes $I \subset \mathcal{V}$ such that there is no edge between any two nodes $i, j \in I$; i.e., for all $i, j \in I$, $(i, j) \notin \mathcal{E}$.

\begin{definition}
	A maximum independent set (MIS) of a graph $\mathcal{G} = (\mathcal{V}, \mathcal{E})$ is an independent set $I$ with the maximum cardinality.
\end{definition}

With this definition, a maximum independent set of a graph, is an independent set with largest possible number of vertices. The maximum independent set problem, is a combinatorial optimization problem defined over a graph whose solution determines a maximum independent set of the graph. This problem can be formulated as the following integer program
\begin{maxi}|l|
	{\mathbf{z}}{f(\mathbf{z}) = \sum_{i=1}^{|\mathcal{V}|} z_i}{}{}
	\label{eq:mis}
	\addConstraint{z_i + z_j \leq 1}{,\quad}{(i, j) \in \mathcal{E}}
	\addConstraint{z_i \in \{0, 1\}}{,\quad}{i \in \{1,\dots,|\mathcal{V}|\}},
\end{maxi}
where $\mathbf{z} = [z_1,\dots,z_{|\mathcal{V}|}]^T$ denotes an indicator vector whose element $z_i$ is a binary variable indicating whether the corresponding node $i$ belongs to the MIS ($z_i = 1$) or not ($z_i = 0$).

\begin{definition}
	Let $\mathbf{w} \in \mathbb{R}^{|\mathcal{V}|}$ be a weight vector assigned to the vertices of a graph $\mathcal{G}$. A maximum weight independent set (MWIS) of $\mathcal{G}$ is an independent set $I$ of the graph with the maximum total weight.
\end{definition}

The maximum weight independent set problem, is a combinatorial optimization problem defined over a graph whose solution determines a maximum weight independent set of the graph. This problem can be formulated as the following integer program
\begin{maxi}|l|
	{\mathbf{z}}{f(\mathbf{z}) = \sum_{i=1}^{|\mathcal{V}|}w_iz_i}{}{}
	\label{eq:mwis}
	\addConstraint{z_i + z_j \leq 1}{,\quad}{(i, j) \in \mathcal{E}}
	\addConstraint{z_i \in \{0, 1\}}{,\quad}{i \in \{1,\dots,|\mathcal{V}|\}},
\end{maxi}
where $w_i$ is the weight assigned to the vertex $i$.

The maximum independent set problem is a special case of the maximum weight independent set problem where all the weights are set to $1$. Both problems are considered as NP-hard problems and finding an exact solution for these problems is intractable.

\subsection{Shannon Capacity and the Confusability Graph}
\label{subsec:shannon}
A channel is a medium that is used to convey an information signal, from one or several senders (or transmitters) to one or several receivers. In information theory, a channel refers to a theoretical channel model with certain error characteristics. The \textit{channel capacity} of a given channel is the highest information rate that can be achieved with arbitrarily small error probability, and is given by the maximum of the mutual information between the input and output of the channel.

Given a communication channel in which certain signal values can be confused with each other due to the presence of noise, the \textit{Shannon capacity} models the rate of information that can be transmitted through such a channel \cite{shannon}. Suppose each distinct message or signal is represented by a vertex in a graph and two vertices are connected by an edge if and only if their corresponding signals are confused through transmission. Mathematically speaking, a channel conveying the transmitted signal $x$ to the received signal $y$ can be modeled as $p_{\textrm{y}|\textrm{x}}(y|x)$. Let $X \in \{x_1, x_2, \dots, x_n\}$ be the transmitted signal and $Y$ be the received signal. For a noiseless channel we have
\begin{gather}
	p(Y=x_i|X=x_i) = 1,\ i \in \{1, \dots, n\},
\end{gather}
meaning that the received signal $Y$ equals to the transmitted signal $X=x_i$. The above expression for a noisy channel can be written as,
\begin{gather}
	p(Y=x_i|X=x_j) = \epsilon_{ij} \geq 0,\ i,j \in \{1,\dots,n\},
\end{gather}	
\begin{figure}[t]
	\begin{center}
		\centerline{\includegraphics[width=0.2\columnwidth]{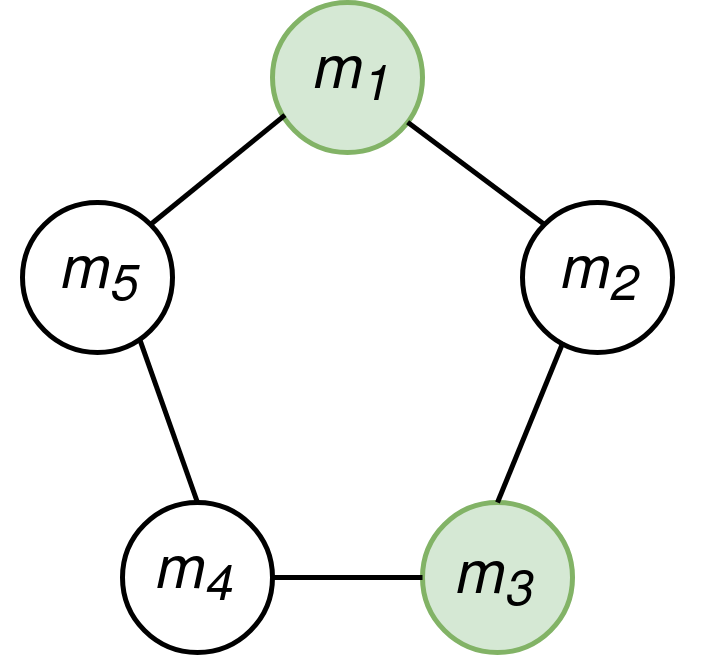}}
		\caption{The confusability graph of a noisy channel with input symbols of $m_1, m_2, m_3, m_4$, and $m_5$, where $m_i$ can be confused with $m_{i+1}$ across the noisy channel and $i$ belongs to a modular arithmetic modulo $5$. The vertices marked as green denote an example for an independent set for which the transmission of corresponding symbols is error-free and without confusion.}
		\label{fig: pentagon}
	\end{center}
\end{figure}
which means that the probability of receiving $Y=x_i$ given that the signal $X=x_j$ is transmitted is $\epsilon_{ij}$. Now, suppose the set $\{x_1, x_2, \dots, x_n\}$ corresponds to the nodes of a graph. Nodes $i$ and $j$ are connected if and only if $\epsilon_{ij} > 0$. The resulting graph is called the \textit{confusability graph} of the noisy communication channel. One can simply observe that sending the symbols belonging to an independent set of the confusability graph yields a reliable and error-free transmission of information through the noisy channel. Figure \ref{fig: pentagon} illustrates a simple example of this concept.

\subsection{Proof of Proposition 1}
\begin{proof}
	Let $X$ and $Y$ be random variables, and $\phi$ be a deterministic differentiable mapping from $X$ to $Y$. We prove the proposition for both cases of $X$ and $Y$ being discrete and continuous. For the discrete case, we have
	\begin{gather}
		\mathcal{I}(X;Y) = H(Y) - H(Y|X),
	\end{gather}
	where $\mathcal{I}(X;Y)$ is the mutual information between $X$ and $Y$ and $H(Y) = -\sum_y p_Y(y)\log p_Y(y)$ is the entropy of the random variable $Y$, and $H(Y|X) = -\sum_y p_{Y|X}(y|x)\log p_{Y|X}(y|x)$ is the conditional entropy of $Y$ given $X$.
	
	Since $Y = \phi(X)$, the conditional probability distribution of $Y|X$ is given by
	\begin{gather}
		p_{Y|X}(y|x) = \left\{\begin{array}{lr}
			1 & y = \phi(x)\\
			0 & \text{otherwise}.\\
		\end{array}
		\right.
	\end{gather}
	Hence, we get
	\begin{gather}\label{eq:p1-dent}
		\begin{split}
			H(Y|X) &= -\sum_y p_{Y|X}(y|x)\log p_{Y|X}(y|x)\\
			&= -p_{Y|X}(\phi(x)|x)\log p_{Y|X}(\phi(x)|x) \\
			&= 0.
		\end{split}
	\end{gather}
	The second line of Eq.(\ref{eq:p1-dent}) is because $p_Y(y)$ is zero everywhere, except where $y = \phi(x)$. Also, the third line of Eq.(\ref{eq:p1-dent}) holds because $\log p_{Y|X}(\phi(x)|x) = \log 1 = 0$. Therefore, we get
	\begin{gather}
		\mathcal{I}(X;Y) = H(Y).
	\end{gather}
	Thus, maximizing $\mathcal{I}(X;Y)$ is equivalent to maximizing $H(Y)$.
	
	For the continuous case, we have
	\begin{gather}\label{eq:p1-diff}
		\mathcal{I}(X;Y) = h(Y) - h(Y|X),
	\end{gather}
	where $h(Y)$ is the \textit{differential entropy} of the continuous random variable $Y$ and is defined as
	\begin{gather}
		h(Y) = -\int f_Y(y)\log f_Y(y)dy,
	\end{gather}
	where $f_Y(y)$ is the probability density function of $Y$. Similarly, $h(Y|X)$ is the conditional differential entropy of $Y|X$ defined as
	\begin{gather}\label{eq:p1-cent}
		h(Y|X) = -\int f_{Y|X}(y|x)\log f_{Y|X}(y|x)dy.
	\end{gather}
	Since $Y = \phi(X)$, the conditional cumulative distribution function of $Y$ given $X$ is given by
	\begin{gather}
		F_{Y|X}(y|x) = p_{Y|X}(Y \leq y| x) = \left\{\begin{array}{lr}
			1 & y \geq \phi(x)\\
			0 & \text{otherwise}.\\
		\end{array}
		\right.
	\end{gather}
	Therefore, the conditional probability density function of $Y$ given $X$ will be
	\begin{gather}
		f_{Y|X}(y|x) = \frac{\partial F_{Y|X}(y|x)}{\partial y} = \delta(y - \phi(x)),
	\end{gather}
	where $\delta(.)$ is the continuous Dirac delta function. Thus, the conditional differential entropy of Eq.(\ref{eq:p1-cent}) will be
	\begin{gather}
		h(Y|X) = -\int \delta(y - \phi(x))\log \delta(y - \phi(x))dy,
	\end{gather}
	which is the differential entropy of shifted Dirac delta function which is $-\infty$. Intuitively, with the mapping from $X$ to $Y$ assumed to be deterministic, the conditional differential entropy $h(Y|X)$ attains its lowest possible value and diverges to $-\infty$. This result is due to the differential nature of the entropy of a continuous random variable.
	
	Assuming that the mapping $\phi$ is parameterized by the parameters $\theta$, based on Eq.(\ref{eq:p1-diff}), we get
	\begin{gather}
		\frac{\partial \mathcal{I}(X;Y)}{\partial \theta} = \frac{\partial h(Y)}{\partial \theta},
	\end{gather}
	because the conditional differential entropy $h(Y|X)$ is independent of $\theta$. This indicates that maximizing the differential entropy of $Y$ is equivalent to maximizing the mutual information between $Y$ and $X$, with both maximizations being performed with respect to the parameters $\theta$ of the mapping \cite{haykin, infomax}.
	
	Since a neural network is a deterministic differentiable mapping, by replacing $Y$ with $X_\mathcal{U}$, the proposition is proven.
\end{proof}
\subsection{Proof of Proposition 2}
\begin{proof}
	Here, we prove the proposition for the discrete random variables and the proof for the continuous case is carried out in a similar way.
	
	Let $\mathcal{G} = (\mathcal{V}, \mathcal{E})$ be an undirected simple graph and $\mathbf{X}$ be a set of random variables assigned to the vertices of $\mathcal{G}$. Also, let $g$ be a probability distribution associated to $\mathbf{X}$ such that for any $(i, j) \notin \mathcal{E}$, it follows that $\mathbf{X}_i$ and $\mathbf{X}_j$ are statistically independent, where $\mathbf{X}_i$ and $\mathbf{X}_j$ denote the $i^\textrm{th}$ and $j^\textrm{th}$ elements of the set $\mathbf{X}$, respectively. We have $H(\mathbf{X}_i) = -\sum g(\mathbf{X}_i)\log g(\mathbf{X}_i)$. 
	Using the probability chain rule, we can compute the joint entropy of $\mathbf{X} = \{\mathbf{X}_1,\dots,\mathbf{X}_{|\mathcal{V}|}\}$ as
	\begin{gather}
		H(\mathbf{X}) = \sum_{i=1}^{|\mathcal{V}|} H(\mathbf{X}_i|\mathbf{X}_1,\dots,\mathbf{X}_{i-1}).
	\end{gather}
	For a set of independent random variables, their joint probability distribution is equal to the product of their marginal probability distributions. Accordingly, the joint entropy is equal to the sum of the entropies of each random variable. Since, in problem \ref{eq:mewis2}, the joint entropy is being maximized over an independent set of the graph $\mathcal{G}$, where the distribution $g$ implicates statistical independence between disconnected vertices, we have
	\begin{gather}
		H(\mathbf{X}_\mathcal{U}) = \sum_{i \in \mathcal{U}} H(\mathbf{X}_i),
	\end{gather} 
	where $\mathcal{U}$ is an independent set of the graph $\mathcal{G}$.
\end{proof}

\subsection{Proof of Proposition 4}
\begin{proof}
	We assume that the reader is familiar with the concept of Markov Random Fields (MRF) and their properties.
	\begin{definition}
		A set of random variables $X$ over a graph $\mathcal{G}=(\mathcal{V}, \mathcal{E})$, is said to be Gibbs Random Field (GRF), if the joint probability distribution can be written as
		\begin{gather}
			p(X)=\frac{1}{Z}\prod_{c_i \in \mathcal{C}}\phi_i(c_i),
		\end{gather} 
		where $\mathcal{C}$ is the set of the cliques of the graph $\mathcal{G}$, and $\phi_i(c_i)$ is the potential assigned to the clique $c_i \in \mathcal{C}$. Also, $Z$ is a normalizing constant which is called the partition function. 
	\end{definition}
	For a GRF, it is typically convenient to write the joint probability as 
	\begin{gather}
		p(X)=\frac{1}{Z}\exp[-\sum_{c_i \in \mathcal{C}}f_i(c_i)],
	\end{gather} 	
	where the potential of each clique is defined as an exponential function.
	\begin{lemma}\label{lemma:prop4}
		If the node probabilities $p(X_i)$ of a set of random variables $X=\{X_1, \dots, X_n\}$ over the nodes of a graph $\mathcal{G}=(\mathcal{V}, \mathcal{E})$ is of exponential form, then the joint probability distribution can be written as a Gibbs Random Field.
	\end{lemma}
	\begin{proof}
		We start backwards, with all the directions being reversible. Suppose $X$ is of Gibbs Random Field form. Then we have
		\begin{gather}
			p(X)=\frac{1}{Z}\prod_{c_i \in \mathcal{C}}\phi_i(c_i)=\frac{1}{Z}\exp[-\sum_{c_i \in \mathcal{C}}f_i(c_i)].
		\end{gather}
		\begin{figure}[t]
			\begin{center}
				\includegraphics[scale=0.06]{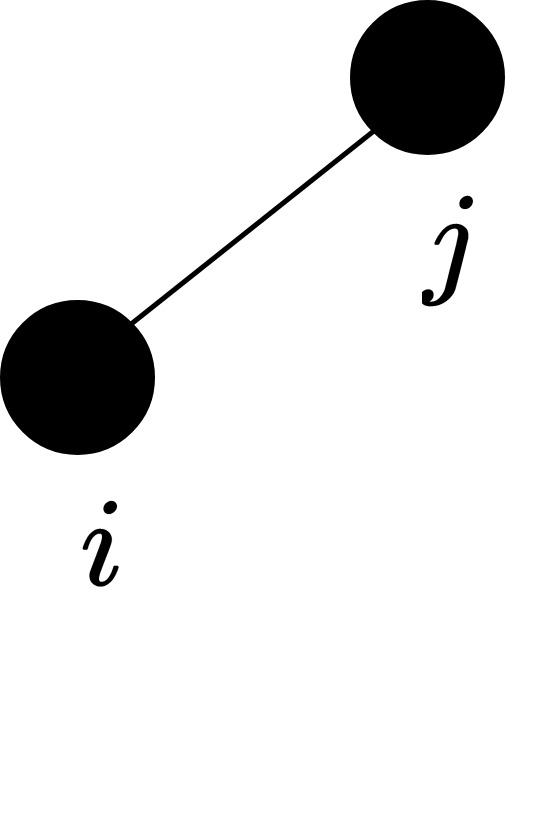}
				\caption{A clique of size 2 that contains nodes $i$ and $j$.}
				\label{fig:clique}
			\end{center}
		\end{figure} 	
		Then, for a clique with two members (Figure \ref{fig:clique}), the joint probability distribution is proportional to the exponential form
		\begin{gather}
			p_{ij} \propto \exp(-\zeta(i, j)),
		\end{gather}
		Where $\zeta$ is a function defined on the clique $\{i, j\}$. In this case, one can simply show that the marginal distributions $p_i$ (or $p_j$) is of the exponential form:
		\begin{gather}
			p_i \propto \exp(-\eta(i)),
		\end{gather}	
		because we can write
		\begin{gather}
			\begin{aligned}
				p_i & = \sum_{j=j_1}^{j_m}p_{ij}=\sum_{j=j_1}^{j_m}\exp(-\zeta(i, j))\\
				& = \exp(-\zeta(i, j_1))[1+ \frac{\exp(-\zeta(i, j_2))}{\exp(-\zeta(i, j_1))}+\dots+\frac{\exp(-\zeta(i, j_m))}{\exp(-\zeta(i, j_1))}],
			\end{aligned}
		\end{gather}
		where $\{j_1, \dots, j_m\}$ is the set of possible values for the random variable defined on the node $j$. Thus we know that the marginal distribution is of exponential form. The procedure can be followed in reverse, i.e., if the marginal distributions are of exponential form, then the joint probability distribution and clique distribution are of exponential form.
	\end{proof}
	The \textit{Hammersly-Clifford} theorem \cite{statsdict} provides the equivalency of a Gibbs Random Field and a Markov Random Field. Thus, according to lemma \ref{lemma:prop4} and Hammersly-Clifford theorem, if we can write the nodes probability distributions in the exponential form, then the joint probability can be written as a Markov Random Field. 
	
	Given that the node probability distributions are defined in exponential and they are constructed given the values of the random variables of the neighboring nodes, i.e., 
	\begin{gather}
		p(\mathbf{X}_i) \propto \exp(-\Delta_i)=\exp(-\eta(\mathbf{X}_i, \mathbf{X}_{\mathcal{N}_i})),
	\end{gather}
	one can show that the joint probability distribution is of Gibbs Random Field, based on Lemma \ref{lemma:prop4}, and based on Hammersly-Clifford theorem, the assigned probability distribution to the graph constitutes a Markov Random Field. Since the node probabilities are defined given the neighboring values, one can immediately conclude that two pairs of disconnected nodes are statistically independent.
\end{proof}

\subsection{Proof of Proposition 5}
\begin{lemma}\label{lem:1}
	Let $\mathcal{G} = (\mathcal{V}, \mathcal{E})$ be a connected simple graph, and $\mathcal{U}$ be an arbitrary maximal independent set of $\mathcal{G}$. For any node $i \in \mathcal{U}$, there is a node $j \in \mathcal{U}$ with the maximum geodesic distance, $d_\mathcal{G}(i, j)$ of $3$ in $\mathcal{G}$.
\end{lemma}
\begin{proof}
	If $i,j \in \mathcal{U}$, then, since $i$ and $j$ are not connected, trivially we have $d_\mathcal{G}(i, j) > 1$.\\
	Now, let $i, j \in \mathcal{U}$, and there is a path of length $4$ between $i$ and $j$ (Figure \ref{fig:g1}). We can rapidly conclude that the nodes $m$ and $n$ do not belong to $\mathcal{U}$ since they are connected to two nodes of $\mathcal{U}$. If the node $w$ belongs to the maximal independent set $\mathcal{U}$ the proof is over since the geodesic distance between the node $i$ and $w$ would be $2$ (through $i \rightarrow m \rightarrow w$). Thus, let's suppose that $w$ does not belong to $\mathcal{U}$.
	
	\begin{figure}[!ht]
		\begin{center}
			\includegraphics[scale=0.15]{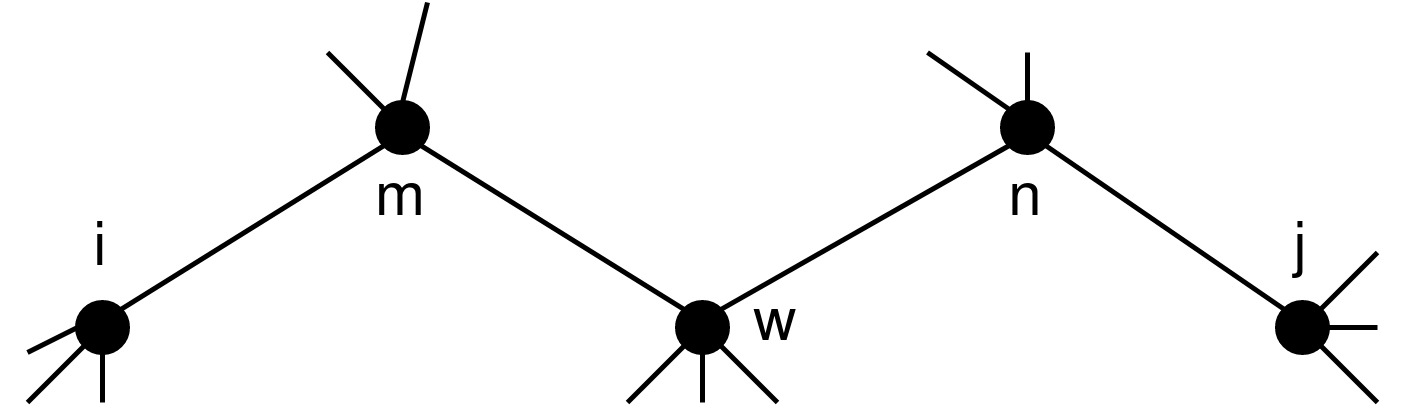}
			\caption{A segment of a graph where nodes $i$ and $j$ belong to the maximal independent set $\mathcal{U}$ and there is a path of length $4$ between them}
			\label{fig:g1}
		\end{center}
	\end{figure}
	
	If $w \notin \mathcal{U}$, it follows that $w$ is directly connected to a node in $\mathcal{U}$. In this case, three situations might happen:
	\begin{enumerate}[(i)]
		\item $w$ is connected to $i$ (Figure \ref{fig:g3}),
		\begin{figure}[!ht]
			\begin{center}
				\includegraphics[scale=0.15]{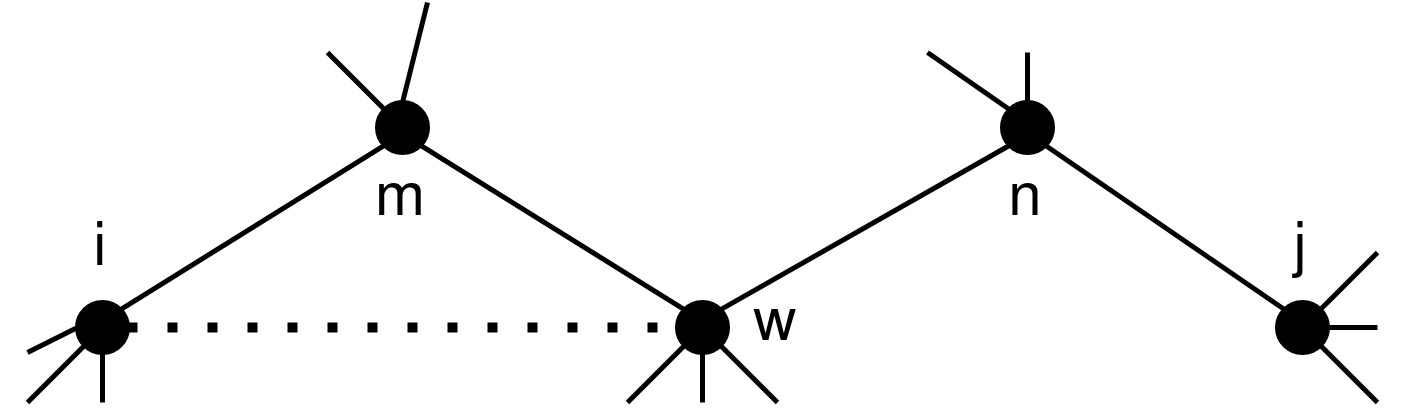}
				\caption{The case where $w$ is connected to $i$.}
				\label{fig:g3}
			\end{center}
		\end{figure}
		\begin{figure}[!ht]
			\begin{center}
				\includegraphics[scale=0.15]{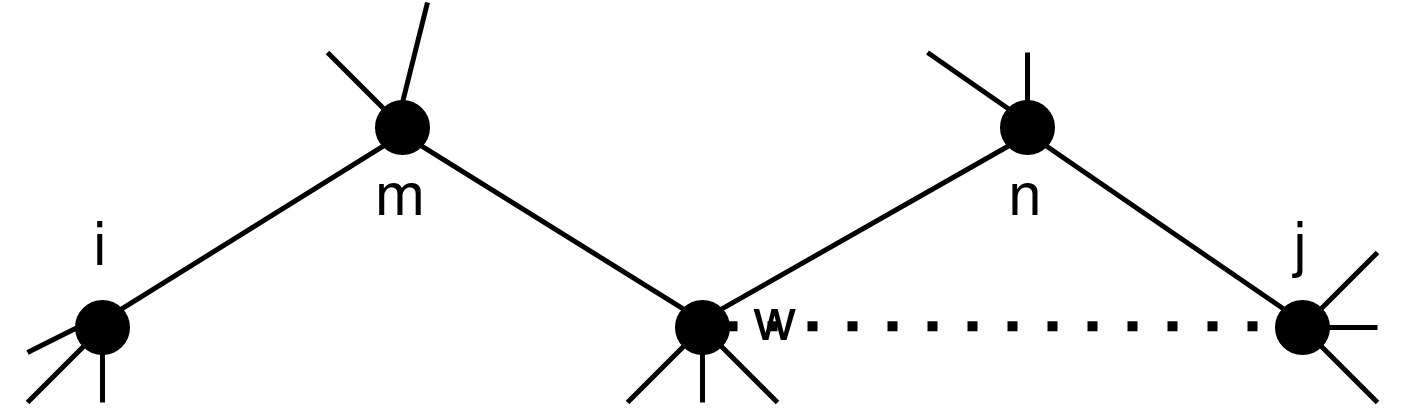}
				\caption{The case where $w$ is connected to $j$.}
				\label{fig:g4}
			\end{center}
		\end{figure}
		\begin{figure}[!ht]
			\begin{center}
				\includegraphics[scale=0.15]{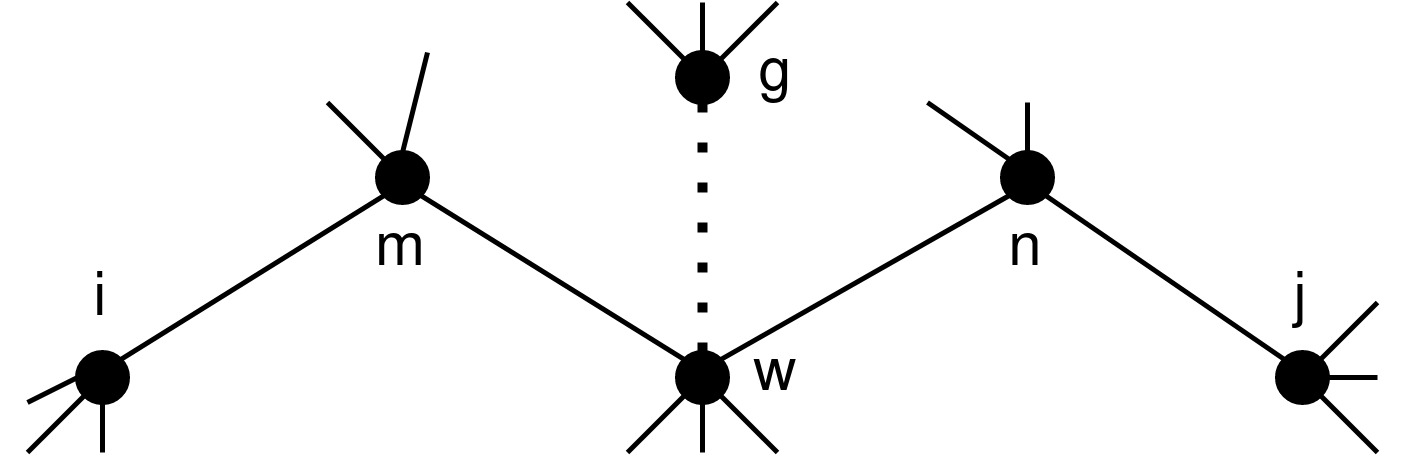}
				\caption{The case where $w$ is connected to another node $g \in \mathcal{U}$.}
				\label{fig:g2}
			\end{center}
		\end{figure}
		\item or, $w$ is connected to $j$ (Figure \ref{fig:g4}),
		\item or, $w$ is directly connected to another node $g \in \mathcal{U}$ (Figure \ref{fig:g2}).
	\end{enumerate}
	
	If case (i) happens, it means that the geodesic distance between $i$ and $j$ is $3$ (through $i \rightarrow w \rightarrow n \rightarrow j$) and the proof is over. If case (ii) occurs, then again, the geodesic distance between $i$ and $j$ is $3$ (through $i \rightarrow m \rightarrow w \rightarrow j$) and the proof is over. Finally, if case (iii) arises, then there is a node $g \in \mathcal{U}$ for which the geodesic distance between $i$ and $g$ is $3$ (through $i \rightarrow m \rightarrow w \rightarrow g$) and the proof is complete.\bigskip\\
	If none of the cases above happens, then $w$ must belong to the maximal independent set $\mathcal{U}$, otherwise the condition of maximality of $\mathcal{U}$ would be violated.
\end{proof}

From the above lemma, we can immediately conclude the following corollary.

\begin{corollary}\label{cor:1}
	Let $\mathcal{G} = (\mathcal{V}, \mathcal{E})$ be a connected simple graph, and $\mathcal{U}$ be an arbitrary maximal independent set of $\mathcal{G}$. For any two nodes $i, j \in \mathcal{U}$, either $d_\mathcal{G}(i, j) \leq 3$ or there is a subset $\mathcal{S} = \{s_1, \dots s_k\}$ of $\mathcal{U}$ such that $d_\mathcal{G}(i, s_1) \leq 3$, $d_\mathcal{G}(s_i, s_{i+1}) \leq 3$, $1 \leq i \leq k-1$, and $d_\mathcal{G}(s_k, j) \leq 3$.\bigskip
\end{corollary}

\begin{proof}
	The case where the path length between $i$ and $j$ is $4$ is shown in the proof of Lemma \ref{lem:1} (Figure \ref{fig:g2}). It is shown that either the shortest path between $i$ and $j$ is at most $3$ or there is a subset $\{g\} \subset \mathcal{U}$ where $d_\mathcal{G}(i, g) \leq 3$ and $d_\mathcal{G}(g, j) \leq 3$. The proof for the cases where the path length between $i$ and $j$ is more than $4$ is carried out similarly.
\end{proof}

Based on Corollary \ref{cor:1}, we can present a theorem which indicates that the reconstruction function preserves the connectivity of the graph.

\begin{prop}
	For a simple graph $\mathcal{G} = (\mathcal{V}, \mathcal{E})$ with an associated adjacency matrix $\mathbf{A}$ and an arbitrary maximal independent set $\mathcal{U}$, the reconstruction function of Eq.(\ref{eq:adj_pooled}), preserves the connectivity of each connected component of the graph.
\end{prop}

\begin{proof}
	Without loss of generality, suppose that $\mathcal{G}$ has only one connected component. If $\mathcal{U}$ is a singleton (it contains only one vertex) then the resulting graph consists of only one vertex and is trivially connected. Now, suppose $\mathcal{U}$ contains $n$ nodes. According to Corollary \ref{cor:1}, for any two nodes in $\mathcal{U}$, there is a sequence of nodes in $\mathcal{U}$ for which the shortest path between each pair of consecutive nodes is at most $3$. Since the reconstruction function $\psi(A, \mathcal{U})$ connects any two nodes of $\mathcal{U}$ for which the shortest path length is at most $3$, it follows that there is path between any two arbitrary nodes in the reconstructed graph and hence, the resulting graph is connected.
\end{proof}
\subsection{Neural Execution of Maximum Weight Independent Set}
\label{subsec:ne_mwis}

In this section we present a neural implementation for the approximate solution of the maximum weight independent set problem. Initially, we briefly introduce the objective function of the MWIS problem. Given a graph $\mathcal{G} = (\mathcal{V}, \mathcal{E})$ and a weight vector $\mathbf{w} \in \mathbb{R}^{|\mathcal{V}|}$, MWIS problem can originally be formulated as an integer programming problem defined as \ref{eq:mwis}.

We can simply replace the first constraint of problem \ref{eq:mwis} with $z_iz_j = 0\ ;\ (i, j) \in \mathcal{E}$. The Lagrangian of this optimization problem can be given as,

\begin{gather}
	\label{eq:2}
	\begin{aligned}
		L(\mathbf{z}, \mathbf{\Lambda}) =& \sum_{i=1}^{|\mathcal{V}|} w_iz_i - \sum_{(i, j) \in \mathcal{E}} \lambda_{ij}z_iz_j,\; \ z_i \in \{0, 1\}.
	\end{aligned}
\end{gather}

As $ z_i \in \{0, 1\}$ we can consider $z_i$ as a Bernoulli random variable with parameter $p_i$. Computing the expected value of Eq.(\ref{eq:2}), gives

\begin{gather}
	\label{eq:3}
	\mathbb{E_{\mathbf{z} \sim \textrm{Ber}(\mathbf{p})}}[L(\mathbf{z}, \mathbf{\Lambda})] = \sum_{i=1}^{|\mathcal{V}|} w_ip_i - \sum_{(i, j) \in \mathcal{E}} \lambda_{ij}p_ip_j.
\end{gather}

Setting $\lambda_{ij} = 1$, Eq.(\ref{eq:3}) turns into the quadratic polynomial formulation of MWIS problem as proven in \cite{quadraticform}:	
\begin{maxi}|l|
	{\mathbf{z}}{H(\mathbf{z}) = \sum_{i=1}^{|\mathcal{V}|} w_iz_i - \sum_{(i, j) \in \mathcal{E}}z_iz_j}{}{}
	\label{eq:4}
\end{maxi}		
where $\mathbf{z}$ is defined over the hypercube $[0, 1]^n$. We employ $H(\mathbf{z})$ as the objective function of a graph neural network which is supposed to learn the parameters of the Bernoulli distribution.

According to the probabilistic method \cite{probabilisticmethod, erdos1959graph}, as the objective function defined in Eq.(\ref{eq:4}) is the expected value of $f(\mathbf{z})$ in problem \ref{eq:mwis}, we have

\begin{gather}
	\label{eq:5}
	p\left(f(\mathbf{z}) \geq \mathbb{E}[f(\mathbf{z})] = H(\mathbf{z})\right) > 0.
\end{gather}

This assures that there exists a valid solution in the distribution generated by the network which has no less objective value of the original problem than the objective value of the network. Similar to \cite{erdos}, we retrieve this solution using derandomization method of conditional expectation \cite{probabilisticmethod}, elaborated in Section(\ref{subsubsec:ce}).

Since we intend to define a loss function for the network, we simply convert the maximization problem of \ref{eq:4} to a minimization problem. This yields the loss function of the network as

\begin{gather}
	\label{eq:a6}
	\mathcal{L}(\mathcal{G};\theta) = \gamma - \sum_{i=1}^{|\mathcal{V}|} w_iz_i + \sum_{(i, j) \in \mathcal{E}}z_iz_j,
\end{gather}

where $\theta$ denotes the parameters of the network and $\gamma$ is an upper bound for $H(\mathbf{z})$ to ensure positive values for $\mathcal{L}(\mathcal{G};\theta)$. A trivial choice for $\gamma$ is the sum of the weights of all nodes, $\gamma = \sum_{i}w_i$.

\subsubsection{Method of Conditional Expectation}
\label{subsubsec:ce}

Since the probabilistic method is nonconstructive, which only proves the existence of a mathematical object with the desired combinatorial structure, derandomization methods are known to provide deterministic algorithms that guarantee to construct such desired structures in a precise and efficient manner \cite{raghavan}. Method of conditional expectation is one of such derandomization methods, which we use in this paper to retrieve the valid solution of maximum weight independent set problem. More precisely, we use the method of conditional expectation to sequentially select the nodes which belong to the maximum weight independent set.

Briefly, the method of conditional expectation works as follows; given a threshold for the expected value of a random variable, select the outcome which results in an expected value not greater than the threshold. Since the conditional expectation never increases, this ensures that we always reach to a solution which is not worse than the threshold of the expected value. For the maximum weight independent set problem the threshold is defined as Eq.(\ref{eq:6}) which is the loss of the network and the whole procedure can be translated as follows; at each step we select a node whose selection and removal of its neighbors leads to a lower expected value than the given threshold. This procedure is given at Table(\ref{tab:1}). 

\subsection{Statistics of Datasets}
The statistics of the datasets that are used in graph classification tasks and the combinatorial optimization problem of maximum independent set are presented in Table \ref{tab:gc_dataset} and Table \ref{tab:mis_dataset}, respectively.
\begin{table*}[ht]
	\centering
	\caption{Datasets statistics for the supervised graph classification task.}
	\resizebox{\textwidth}{!}{
		\begin{tabular}{lcccccc}
			\toprule
			Name & \#Graphs & Avg. \#Nodes & Avg. \#Edges & \#Classes & \#Node Features & Node Label\\
			\midrule
			D\&D\cite{dd1prot, dd2} & $1178$ & $284.32$ & $715.66$ & $2$ & - & +\\
			FRANKENSTEIN\cite{frankenstein} & $4337$ & $16.90$ & $17.88$ & $2$ & $780$ & - \\
			MUTAG\cite{mutag1, mutag2} & $188$ & $17.93$ & $19.79$ & $2$ & - & +\\
			Mutagenicity\cite{mutagenicity1, mutagenicity2} & $4337$ & $30.32$ & $30.77$ & $2$ & - & +\\
			PROTEINS\cite{dd1prot, enzprot} & $1113$ & $39.06$ & $72.82$ & $2$ & $1$ & + \\
			COLLAB\cite{imdbcollab} & $5000$ & $74.49$ & $2457.78$ & $3$ & - & -\\
			IMDB-BINARY\cite{imdbcollab} & $1000$ & $19.77$ & $96.53$ & $2$ & - & -\\
			IMDB-MULTI\cite{imdbcollab} & $1500$ & $13.00$ & $65.94$ & $3$ & - & -\\
			\bottomrule				
	\end{tabular}}
	\label{tab:gc_dataset}
\end{table*}
\begin{table}[ht]
	\centering
	\caption{Datasets statistics for the unsupervised combinatorial optimization problem of maximum independent set.}
	\resizebox{0.4\textwidth}{!}{
		\begin{tabular}{lcc}
			\toprule
			Name & \#Nodes & \#Edges \\
			\midrule
			Cora\cite{citationnetworks} & $2708$ & $5429$ \\
			Citeseer\cite{citationnetworks} & $3327$ & $4732$ \\
			PubMed\cite{citationnetworks} & $19717$ & $44338$ \\
			\bottomrule
	\end{tabular}}
	\label{tab:mis_dataset}
\end{table}

\subsection{Experimental Settings}

The trade-off between the cross-entropy loss for classification and the MEWISPool’s loss are controlled via a regularization factor which is set to $0.01$ for all the graph classification experiments. In this section we investigate the effect of the regularization factor on the performance of the model in terms of validation accuracy for model selection. The experiments are conducted on the \textbf{IMDB-BINARY} dataset. We sweep through the values in $\{1,0.1,0.01,0.001,0.0001\}$ and record the validation accuracy. The result is illustrated in Figure \ref{fig:reg-fac}. Note that the values for the regularization factor are represented in log-scale in the basis of $10$.  As it is shown, the best validation accuracy is achieved by setting the regularization factor to $0.01$.
\begin{figure}[!ht]
	\begin{center}
		\includegraphics[scale=0.15]{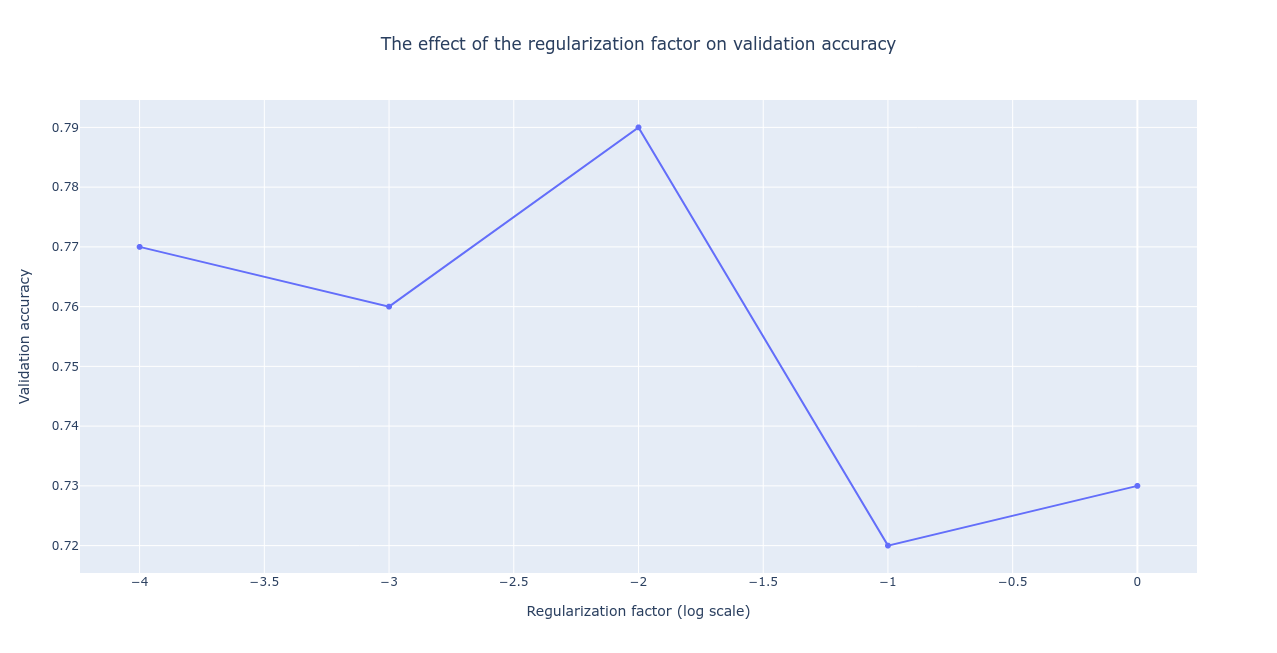}
		\caption{The effect of regularization factor on the validation accuracy. As the figure shows, the best accuracy is achieved for the regularization factor set to $0.001$.}
		\label{fig:reg-fac}
	\end{center}
\end{figure}

Furthermore, we depict the learning curve of the model in terms of separate losses for training. Figure \ref{fig:learn-curve} demonstrates the learning curves corresponding to the cross-entropy loss for classification, MEWISPool loss, and the total loss with the regularization factor taken into account.
\begin{figure}[!ht]
	\begin{center}
		\includegraphics[scale=0.15]{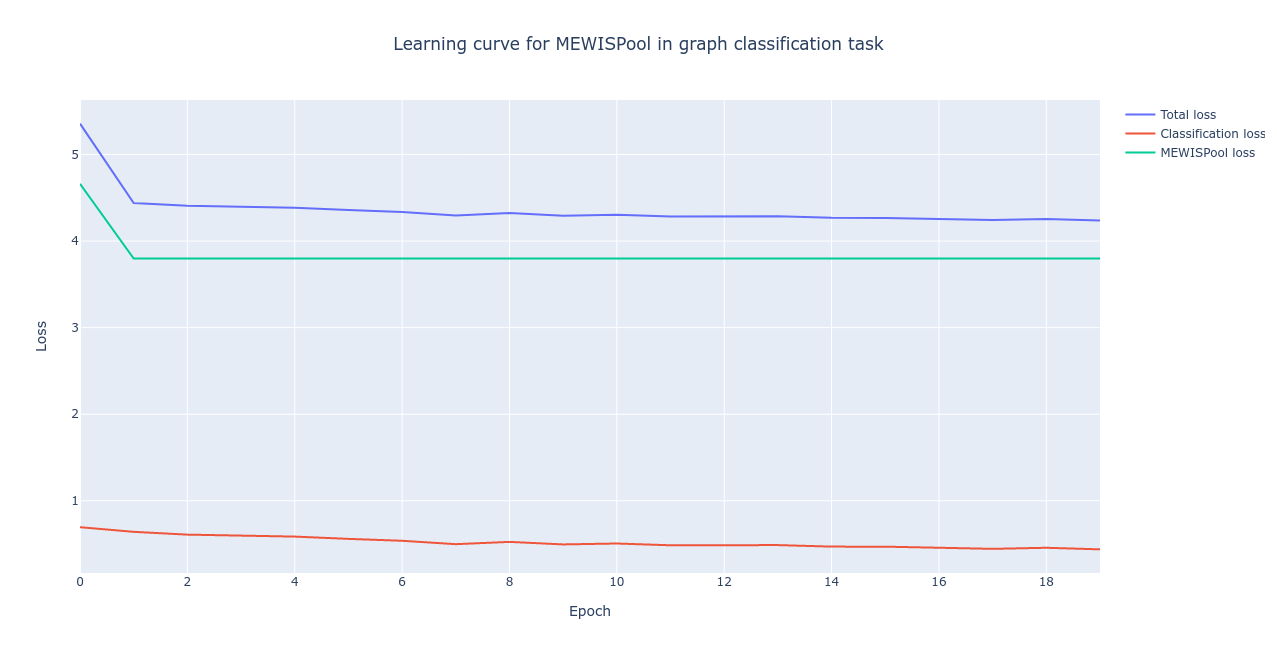}
		\caption{The learning curves corresponding to the cross-entropy loss, MEWISPool loss, and the total loss.}
		\label{fig:learn-curve}
	\end{center}
\end{figure}
The architecture used for the graph classification task is shown in Figure \ref{fig:arch}. In this architecture the MEWISPool consists of three graph convolutional layers. All the convolutional layers are graph isomorphism networks.
\begin{figure}[!ht]
	\begin{center}
		\includegraphics[scale=0.15]{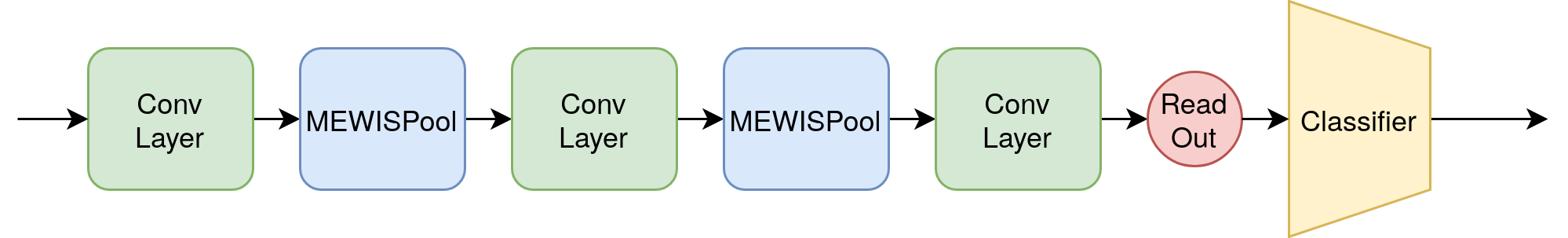}
		\caption{The architecture used for the graph classification task.}
		\label{fig:arch}
	\end{center}
\end{figure}

\end{document}